\documentclass{article}
\usepackage{ijcai19}

\usepackage{times}
\usepackage{soul}
\usepackage{url}
\usepackage[hidelinks]{hyperref}
\usepackage[utf8]{inputenc}
\usepackage[small]{caption}
\usepackage{graphicx}
\usepackage{amsmath}
\usepackage{amssymb}
\usepackage{booktabs}
\usepackage{algorithm}
\usepackage{algorithmic}
\usepackage{array}
\usepackage{lipsum}
\newtheorem{theorem}{Theorem}

\newtheorem{proof}{Proof}
\title{Policy Optimization With Penalized Point Probability Distance: An Alternative To Proximal Policy Optimization}

\author{
    Xiangxiang Chu
    \affiliations
    AI Lab, Xiaomi \emails
    chuxiangxiang@xiaomi.com
}
\begin{document}

\maketitle

\begin{abstract}
 As the most successful variant and improvement for Trust Region Policy Optimization (TRPO), proximal policy optimization (PPO) has been widely applied across various domains with several advantages: efficient data utilization, easy implementation, and good parallelism. In this paper, a first-order gradient reinforcement learning algorithm called Policy Optimization with Penalized Point Probability Distance (POP3D), which is a lower bound to the square of total variance divergence is proposed as another powerful variant. Firstly, we talk about the shortcomings of several commonly used algorithms, by which our method is partly motivated. Secondly, we address to overcome these shortcomings by applying POP3D. Thirdly, we dive into its mechanism from the perspective of solution manifold. Finally, we make quantitative comparisons among several state-of-the-art algorithms based on common benchmarks. Simulation results show that POP3D is highly competitive compared with PPO. Besides, our code is released in \url{https://github.com/paperwithcode/pop3d}.
  
\end{abstract}

\section{Introduction}

With the development of deep reinforcement learning, lots of impressive results have been produced in a wide range of fields such as playing Atari game \cite{mnih2015human,hessel2017rainbow},  controlling robotics \cite{lillicrap2015continuous}, Go \cite{silver2017mastering}, neural architecture search \cite{tan2018mnasnet,pham2018efficient}. 


The basis of a reinforcement learning algorithm is generalized policy iteration \cite{sutton2018reinforcement}, which states  two essential iterative steps: policy evaluation and improvement. Among various algorithms, policy gradient is an active branch of reinforcement learning whose foundations are  $Policy$ $Gradient$ $Theorem$ and the most classical  algorithm REINFORCEMENT \cite{sutton2018reinforcement}. Since then, handfuls of policy  gradient variants have been proposed, such as  Deep Deterministic Policy Gradient (DDPG) \cite{lillicrap2015continuous}, Asynchronous Advantage Actor Critic (A3C) \cite{2016arXiv160201783M}, Actor Critic using Kronecker-factored Trust Region (ACKTR) \cite{wu2017scalable}, Proximal Policy Optimization (PPO) \cite{2017arXiv170706347S}.

Improving the strategy monotonically had been nontrivial before the trust region policy optimization(TRPO) was proposed \cite{2015arXiv150205477S}. Hessian-free strategy: Fisher vector product is utilized to cut down the computing burden. Specifically, Kullback–Leibler divergence(\textbf{KLD})  acts as a hard constraint in place of objective, because its corresponding coefficient is difficult to set for different problems.  However, TRPO still has several drawbacks: too complicated, inefficient data usage. Quite a lot of  efforts have been devoted to improving TRPO since then and the most distinguished one is PPO.

PPO can be regarded as a  first-order variant of TRPO and  have obvious improvements in several facets. In particular, a pessimistic clipped surrogate objective is proposed where TRPO's hard constraint is replaced  by the clipped action probability ratio. In such a way, it constructs an unconstrained optimization problem so that any first-order stochastic gradient optimizer can be  directly applied. Besides, it's easier to be  implemented, more robust against various problems and  achieves an impressive result on Atari games\cite{brockman2016openai}.

Unfortunately, no further analysis of PPO's success is provided in that paper. In fact, the pessimistic surrogate objective which is the most critical component of PPO still has some limitations. 

As another potential improvement for TRPO and alternative to PPO, this paper focuses on a policy optimization algorithm, where its contributions are:
\begin{enumerate}
	\item It proposes a simple variant of TRPO called POP3D
	along with a new surrogate objective containing a point
	probability penalty item, which is symmetric lower
	bound to the square of the total variance divergence for
	two policy distributions. Specifically, it helps to stabilize the learning process and encourage exploration. Furthermore, it escapes from penalty item setting headache
	along with penalized version TRPO, where is arduous to
	select one fixed value for various environments.
	\item It achieves state-of-the-art results with a clear margin on
	49 Atari games within 40 million frame steps based on
	two shared metrics. Moreover, it also achieves competitive results compared with PPO in the continuous domain.
	\item It dives into the mechanism of PPO's improvement over TRPO by the perspective of solution manifold, which also plays an important role in our method.
	\item It enjoys almost all PPO's advantages such as easy implementation, fast learning ability.
\end{enumerate}



\section{Background and Related Works}
\subsection{Policy Gradient}
Agents interact with the environment and receive rewards  which are used to adjust their policy in turn. At state ${s_t}$, one agent takes strategy $\pi$ and transfers to  a new state ${s_{t+1}}$, rewarded ${r_t}$ by the environment. Maximizing discounted return (accumulated rewards) $R_t$ is its objective. In particular, given a policy $\pi$,  $R_t$ is defined as
\begin{equation}\label{R}
R_t=\sum_{n=0}^{\infty} (r_t+\gamma{r_{t+1}}+\gamma^2{r_{t+2}}+...+\gamma^n{r_{t+n}}).
\end{equation}

$\gamma$ is the discounted coefficient to control future rewards, which lies in the range (0, 1). Regarding a neural network with parameter $\theta$, the policy $\pi_\theta(a|s)$ can be learned by maximizing Equation~\ref{R} using the back-propagation algorithm. Particularly, given $Q{(s,a)}$ which represents the agent's return in state $s$ after taking action $a$, the objective function can be written as
\begin{equation}\label{policy gradient theorem}
\max_\theta \quad E_{s,a} \log \pi_\theta(a|s) Q(s,a).
\end{equation}

Equation~\ref{policy gradient theorem} lays the foundation for handfuls of policy gradient based algorithms,
Another variant  can be deduced by using  
\begin{equation}\label{advantage}
A(s,a) = Q(s,a)-V(s)
\end{equation}
to replace $Q(s,a)$ in Equation~\ref{policy gradient theorem} equivalently, $V(s)$ can be any function so long as  $V$ depends on   $s$ but not $a$. In most cases, state value function is used for $V$, which not only helps to reduce variations but has clear physical meaning. Formally, it can be written as
\begin{equation}\label{policy gradient theorem with advantage}
\max_\theta \quad E_{s,a} \log \pi_\theta(a|s) A(s,a).
\end{equation}

\subsection{Advantage Estimate}
One commonly used method for advantage calculation is one step estimation, which estimates 
\begin{equation}
\begin{split}
A(s_t,a) &= Q(s_t,a) - V(s_t) \\
&=r_t + \gamma V(s_{t+1})-V(s_t).
\end{split}
\end{equation}

A better estimate for advantage called generalized advantage estimation is proposed in \cite{2015arXiv150602438S}, where  one, two, three, up to $\infty$ time step estimate are combined and summarized  using $\lambda$ based weights, which helps to estimate more accurately. The generalized advantage estimator is defined as
\begin{equation}
\begin{split}
\hat{A}_t^{GAE(\gamma,\lambda)} &= \sum_{l=0}^{\infty}(\gamma\lambda)^l\delta_{t+l}^V \\
\delta_{t+l}^V &=r_{t+l} + \gamma V(s_{t+l+1})-V(s_{t+l}). 
\end{split}
\end{equation}

The parameter $\lambda$ meets $0\leq\lambda\leq1$, which controls the trade-off between bias and variance.
\subsection{Trust Region Policy Optimization} \label{TRPO_subsection}
Schulman \textit{et al.} propose TRPO  to  update the policy monotonically. In particular, its mathematical form is \\
\begin{equation}\label{Eq1_trpo}
\begin{split}
\max_\theta \quad  & E_t [ \frac{\pi_\theta(a_t|s_t)}{\pi_{\theta_{old}}(a_t|s_t)} \hat{A_t} ] \\
- &C E_t[KL[\pi_{\theta_{old}}(\cdot|s_t), \pi_\theta(\cdot|s_t)]] \\
C & = \frac{2\epsilon\gamma}{(1-\gamma)^2} \\
\epsilon &= \max_s E_{a\sim\pi_{\theta}(a|s)}[A_{\pi_{\theta_{old}}}(s,a)]), 
\end{split}
\end{equation}
where $C$ is the penalty coefficient. 

In practice, the policy update steps would be too small if $C$ is valued as Equation~\ref{Eq1_trpo}. In fact, it's intractable to calculate $C$  beforehand since it requires traversing all states to reach the maximum. Moreover, inevitable bias and variance will be introduced by estimating the advantages of old policy while training. Instead, a surrogate objective is maximized based on the KLD constraint between the old and new policy, which can be written as below,

\begin{equation}\label{Eq2_trpo_constraint}
\begin{split}
\max_\theta \quad  & E_t [ \frac{\pi_\theta(a_t|s_t)}{\pi_{\theta_{old}}(a_t|s_t)} \hat{A_t} ]\\
s.t.\quad   & E_t[KL[\pi_{\theta_{old}}(\cdot|s_t), \pi_\theta(\cdot|s_t)]] \le \delta 
\end{split}
\end{equation}
where $\delta$ is the KLD upper limitation.  In addition, the conjugate gradient algorithm is applied to solve Equation~\ref{Eq2_trpo_constraint} more efficiently. Two major problems have yet to be addressed: one is its complexity even using the conjugate gradient approach, another is compatibility  with architectures that involve noise or parameter sharing tricks \cite{2017arXiv170706347S}.
\subsection{Proximal Policy Optimization}
To overcome the shortcomings of TRPO, PPO replaces the original constrained problem  with a pessimistic clipped surrogate objective where KL constraint is implicitly imposed. The loss function can be written as
\begin{equation}\label{ppo_eq}
\begin{split}
L^{CLIP}(\theta) & =  E_t[ \min(r_t(\theta)\hat{A_t}, \\
&clip(r_t(\theta), 1-\epsilon, 1+\epsilon)\hat{A_t})] \\
&r_t(\theta) = \frac{\pi_\theta(a_t|s_t)}{\pi_{\theta_{old}}(a_t|s_t)}_,
\end{split}
\end{equation}
where $\epsilon$ is a hyper-parameter to control the clipping ratio.
Except for the clipped PPO version, KL penalty versions including fixed and adaptive KLD.  Besides, their simulation results convince that clipped PPO performs best with an obvious margin across various domains.

\section{Policy Optimization with Penalized Point Probability Distance}
Before diving into the details of POP3D,  we review some drawbacks of several methods, which partly motivate us.

\subsection{Disadvantages of Kullback-Leibler Divergence}

According to \textbf{Theorem 1} in TRPO paper, the following inequality holds.
\begin{equation}
\begin{split}
\eta(\pi_\theta) & \le  L_{\pi_{\theta_{old}}}(\pi_\theta)+\frac{2\epsilon\gamma}{(1-\gamma)^2}\alpha^2\\
\alpha &= D_{TV}^{\max}(\pi_{\theta_{old}}, \pi_{\theta})\\
D_{TV}^{\max}(\pi_{\theta_{old}}, \pi_\theta) &= \max_s D_{TV}(\pi_{\theta_{old}}|| \pi_\theta)
\end{split}
\end{equation}

TRPO replaces the square of total variation divergence $D_{TV}^{max}(\pi_{\theta_{old}},\pi_{\theta})$ by $D_{KL}^{\max}(\pi_{\theta_{old}},\pi_\theta) = \max_s D_{KL}(\pi_{\theta_{old}}||\pi_{\theta})$. Given a discrete distribution $p$ and $q$, their total variation divergence $D_{TV}(p||q)$ is defined by $\frac{1}{2}\sum_i|p_i-q_i|$. Obviously, $D_{TV}$ is symmetric by definition, while KLD is  asymmetric.

Formally, given state $s$, KLD of $\pi_{\theta_{old}}(\cdot|s)$  for  $\pi_{\theta}(\cdot|s)$  can be written as
\begin{equation}\label{DPP}
\begin{split}
D_{KL}(\pi_{\theta_{old}}(\cdot|s)||\pi_{\theta}(\cdot|s))&=\sum_{a}{}\pi_{\theta_{old}}(a|s)\\
&\ln\frac{\pi_{\theta_{old}}(a|s)}{\pi_{\theta}(a|s)}.
\end{split}
\end{equation}

Similarly, KLD in the continuous domain can be defined simply by replacing summation with integration.

The consequence of KLD's asymmetry leads to a non-negligible difference of whether choose $D_{KL}(\pi_{\theta_{old}}||\pi_{\theta})$ or $D_{KL}(\pi_{\theta}||\pi_{\theta_{old}})$. Sometimes, those two choices result in quite different solutions. Robert compared the forward and reverse KL on a distribution, one solution matches only one of the modes, and another covers both modes \cite{robert2014machine}. Therefore, KLD is not an ideal bound or approximation for the expected discounted cost. 

\subsection{Discussion about Pessimistic Proximal Policy}
In fact, PPO is called pessimistic proximal policy optimization\footnote{The word "pessimistic" is  used by the PPO paper.} in the meaning of its objective construction style.


Without loss of generality, supposing ${A_t} > 0$  for given state $s_t$ and action $a_t$, and the optimal choice is $a_t^\star$. When $a_t=a_t^\star$, a good update policy  is to increase the probability of action to a relatively high value $a_t^\star$ by adjusting $\theta$. However, the clipped item $clip(r_t(\theta), 1-\epsilon, 1+\epsilon)\hat{A_t}$ will fully contribute to the loss function by the minimum operation, which ignores further reward by zero gradients even though it's the optimal action. Other situation with ${A_t} < 0$ can be analyzed in the same manner.

However, if the pessimistic limitation is removed, PPO's performance decreases dramatically \cite{2017arXiv170706347S}, which is again confirmed by our preliminary experiments. In a word, the pessimistic mechanism plays a very critical role for PPO by a relatively weak preference for good action decision for a given state, which in turn affects learning efficiency.

\subsection{Restricted Solution Manifold}
To be simple, we don't take the model identifiability issues along with deep neural network into account here because they don't affect the following discussion much \cite{goodfellow2016deep}. Suppose $\pi_{\theta\star}$ is the optimal solution for a given environment, in most cases, more than one parameter set for $\theta$ can generate the ideal policy, especially when $\pi_{\theta\star}$ is learned by a deep neural network. In other words, the relationship between $\theta$ and $\pi_{\theta\star}$ is many to one. On the other hand, when agents interact with the environment using policy represented by neural networks, the action is taken approximately strongly corrected with the highest probability value. Although some strategies of enhancing exploration are applied, they don't affect the policy much in the meaning of expectation. 

Taking the Atari-Pong game for example, when an agent sees a Pong ball is coming nearly, its optimal policy is moving the racket to  the right position. The probability of this action is a relatively high value such as 0.95 and it's near to impossible that this value is 1.0 since it's produced by a softmax operation on the several discrete actions. In fact, we hardly obtained the optimal solution accurately, instead, our goal is a good enough answer. Namely, $\theta_1$ outputting a probability 0.95 and $\theta_2$ with 0.9 for the right action are both good answers. During the training process, these similar events occur frequently.

Using a penalty such as KLD cannot handle it effectively, because it involves all of the actions' probabilities. Moreover, it doesn't stop penalizing unless two distributions become exactly indifferent or the advantage item is large enough to compensate for the KLD cost. Therefore, even if $\theta$ outputs $\theta_{old}$ the same high probability for the right action, it's still penalized owing to probabilities mismatch for other uncritical actions. Indeed, when a person is asked to make the choice, corresponding action will be taken only if the probability is above a threshold. From the perspective of the manifold, if the optimal parameters constitute a solution manifold. The KLD penalty will act until $\theta$ exactly locates in the solution if possible. However, if the agent concentrates only on critical actions like a human, it's much easier to approach the manifold, which in fact, expands the solution manifold at least one dimension such as curves to surfaces and surfaces to spheres at best.

%

Besides, since mini-batch is a commonly used trick for training neural networks, removing this unexpected penalty helps to decrease penalty noises, which are reflected by the corresponding gradient. 

\subsection{Exploration}
One shared highlight in reinforcement learning is the balance between exploitation and exploration. For a policy-gradient  algorithm, entropy is added in the total loss to encourage exploration in most case. When included in loss function, KLD penalizes the old and new policy probability mismatch for all possible actions as Equation~\ref{DPP}. This strict punishment for every action's probability mismatch, which discourages exploration.

\subsection{Point Probability Distance}

To overcome the above-mentioned shortcomings, we propose a surrogate objective with the point probability distance penalty, which is symmetric and more optimistic than PPO. In the discrete domain, when the agent takes action $a$, the point probability distance between  $\pi_{\theta_{old}}(\cdot|s)$  and  $\pi_{\theta}(\cdot|s)$   is defined by

\begin{equation}\label{PPD}
\begin{split}
D_{pp}(\pi_{\theta_{old}}(\cdot|s), \pi_{\theta}(\cdot|s))&=({\pi_{\theta_{old}}(a|s)}\\
&-{\pi_{\theta}(a|s)})^2.
\end{split}
\end{equation}

Attention should be paid to the penalty definition item, the distance is measured by the point probability, which emphasizes its mismatch for the \textbf{sampled} actions for a state.   
Undoubtedly, $D_{pp}$ is symmetry by definition. Furthermore, it can be proved that $D_{pp}$ is indeed a lower bound for the total variance divergence $D_{TV}$. As a special case, it can be easily proved that for binary distribution, $D_{TV}^2(p||q) = D_{pp}(p||q)$. 

\begin{theorem}
	For two discrete probability distributions p and q with $K$ values,  then $D_{TV}^2(p||q) \ge D_{pp}(p||q)$ holds.
\end{theorem}
\begin{proof}
	Let  $a=p_l, b=q_l$ for any $l$, and suppose $a\ge b$ without loss of generalization. So,
	\begin{eqnarray*}
		D_{TV}^2(p||q) & = & (\frac{1}{2}\sum_{i=1}^{K}|p_i-q_i|)^2 \\
		& = & (\frac{1}{2}\sum_{i=1,i\ne l}^K|p_i-q_i|+\frac{1}{2}|p_l-q_l|)^2 \\
		& \ge & (\frac{1}{2}|\sum_{i=1,i\ne l}^Kp_i-q_i|+\frac{1}{2}(a-b))^2 \\
		& = & (\frac{1}{2}|1-a-(1-b)| + \frac{1}{2}(a-b)) ^2\\
		& = & (\frac{1}{2}(a-b)+\frac{1}{2}(a-b))^2 \\
		& = & D_{pp}(p||q)
	\end{eqnarray*}
	
\end{proof}
Q.E.D.

Since $0\leq{\pi_{\theta}(a|s)}\leq1$ holds for discrete action space,  $D_{pp}$ has a lower and upper boundary: $0\leq D_{pp}\leq1$. Moreover, $D_{pp}$ is less sensitive to action space dimension than KLD, which has a similar effect as PPO's clipped ratio to increase robustness and enhance stability.\\

Equation~\ref{PPD} stays unchanged for the continuous domain, and the only difference is  $ \pi_{\theta}(a|s)$ represents point probability density instead of probability.
\subsection{POP3D}
After we have defined the point probability distance, we use a new surrogate objective for POP3D, which can be written as 
\begin{equation}\label{POP3DOBJ}
\begin{split}
\max_\theta \quad  & E_t [ \frac{\pi_\theta(a_t|s_t)}{\pi_{\theta_{old}}(a_t|s_t)} \hat{A_t} - \\
\quad & \beta D_{pp}(\pi_{\theta_{old}}(\cdot|s), \pi_{\theta}(\cdot|s))],
\end{split}
\end{equation}
where $\beta$ is the penalized coefficient. These combined advantages  lead to considerable performance improvement, which escapes from the dilemma of choosing preferable penalty coefficient. Besides, we use generalized advantage estimates to calculate $\hat{A_t}$. Algorithm~\ref{pop3d} shows the complete iteration process of POP3D. Moreover, it possesses the same computing cost and data efficiency as PPO. 

\begin{algorithm}[tb]
	\caption{POP3D}
	\label{pop3d}
	\begin{algorithmic}[1]
		\STATE {\bfseries Input:} max iterations $L$ , actors $N$, epochs $K$
		\FOR{$iteration=1$ {\bfseries to} $L$}
		\FOR{$actor=1$ {\bfseries to} $N$}
		\STATE Run policy $\pi_{\theta_{old}}$ for $T$ time steps 
		\STATE Compute advantage estimations $\hat{A}_1,...,\hat{A}_T$
		\ENDFOR
		\FOR{$epoch=1$ {\bfseries to } $K$}
		\STATE Optimized loss objective wrt $\theta$ with mini-batch size $M \leq NT$, then update $\theta_{old}\leftarrow\theta$.
		\ENDFOR
		\ENDFOR
	\end{algorithmic}
\end{algorithm}

\subsection{Relationship With PPO}
To conclude this section, we take some time to see why PPO works by taking the above viewpoints into account. 

When we pour more attention to  Equation~\ref{ppo_eq}, the ratio $r_t(\theta)$ only involves the probability for given action $a$, which is chosen by policy $\pi$. In other words,  all other actions' probabilities except $a$ are not activated, which no longer contribute to back-propagation and allow probability mismatch. Obviously, this procedure behaves similarly as POP3D, which expands the restricted solution manifold.

Above all, POP3D is designed to conform with the regulations for overcoming above mentioned problems, and in the next section experiments from commonly used benchmarks will evaluate its performance.

\section{Experiments}
\subsection{Controlled Experiments Setup}

OpenAI Gym is a well-known simulation environment to test and evaluate various reinforcement algorithms, which is composed of  both discrete (Atari) and continuous (Mujoco) domains \cite{brockman2016openai}. Most of recent deep reinforcement learning  methods such as DQN variants \cite{van2016deep,wang2015dueling,schaul2015prioritized,bellemare2017distributional,hessel2017rainbow}, A3C, ACKTR, PPO are evaluated using only one set of hyper-parameters\footnote{DQN variants are evaluated in Atari environment since they are designed to solve problems about discrete action space. However, policy gradient based algorithms can handle both continuous and discrete problems.}. Therefore,  we evaluate POP3D's performance on 49 Atari games(v4, discrete action space ) and 7 Mujoco(v2, continuous space).

Since PPO is  a distinguished RL algorithm which defeats various methods such as A3C, A2C ACKTR, we focus on a detailed quantitative comparison with fine-tuned PPO. And we don't consider large scale distributed algorithms Apex-DQN \cite{horgan2018distributed} and IMPALA \cite{espeholt2018impala}, because we concentrate on comparable and fair evaluation, while the latter is designed to apply with large scale parallelism. Nevertheless, some orthogonal improvements from those methods have the potentials to  improve our method further. Furthermore, we include TRPO to acts as a baseline method. In addition, quantitative comparisons between KLD and  point probability penalty helps to convince the critical role of the latter, where the former strategy is named fixed KLD in \cite{2017arXiv170706347S} and can act as another good baseline in this context, named by \textbf{BASELINE} below.

In particular, we retrained one agent for each game with fine-tuned hyper-parameters\footnote{
	We use OpenAI's PPO and TRPO implementation, which provides a good baseline for various reinforcement learning algorithms: \url{https://github.com/openai/baselines.git}.}. 
To avoid the problems of reproduction about reinforcement algorithms mentioned in \cite{henderson2017deep}, we take the following measures:

\begin{itemize}
	\item Use the same training steps and make use of the same amount of game frames(40M for Atari game and 10M for Mujoco)\@.
	\item Use the same neural network structures, which is CNN model with one action head and one value head for Atari game, and fully-connected model with one value head and one action head which produces the mean and standard deviation of diagonal Gaussian distribution as PPO\@. 
	\item Initialize  parameters using the  same strategy as PPO\@.
	\item Keep Gym wrappers from  Deepmind such as reward clipping and frame stacking unchanged for Atari domain, and enable 30 no-ops at the beginning of each episode\@.
	\item Use Adam optimizer \cite{kingma2014adam} and decrease  $\alpha$ linearly from 1 to 0 for Atari domain as PPO\@.
\end{itemize}

To facilitate further comparisons with other approaches,  we release the seeds and  detailed results\footnote{Experiment results can be downloaded from \url{https://drive.google.com/file/d/1c79TqWn74mHXhLjoTWaBKfKaQOsfD2hg/view?usp=sharing}.}(across the entire training process for different trials). In addition, we randomly select three seeds from \{0, 10, 100, 1000, 10000\} for two domains, \{10,100,1000\} for Atari and \{0,10,100\} for Mujoco in order to decrease unfavorable subjective bias stated in \cite{henderson2017deep}.
\subsection{Evaluation Metrics}
PPO utilizes two score metrics for evaluating  agents performance using various RL algorithms. One is the mean score of last 100 episodes $Score_{100}$, which measures how high a strategy can hit eventually. Another is the average score across all episodes $Score_{all}$, which evaluates how fast an agent learns. In this paper, we  conform to this routine and calculate individual metric by averaging three seeds in the same way.

\subsection{Discrete Domain Comparisons}
\subsubsection{Hyper-parameters}
We search hyper-parameter four times for the penalty coefficient $\beta$ based on four Atari games while keeping other hyper-parameters unchanged as PPO and fix $\beta=5.0$ to train all Atari games. For BASELINE, we also search hyper-parameter four times on penalty coefficient $\beta$ and choose $\beta=10.0$. To save space, detailed hyper-parameter setting can be found in Table~ \ref{PPO-Atari-game-hyperparameter}, \ref{POP3D-Atari-game-hyperparameter} and  \ref{Baseline-Atari-game-hyperparameter}.

This process is not beneficial for POP3D owing to missing optimization for all hyper-parameters. There are two reasons to make this choice. On the one hand, it’s the simplest way to make a relatively fair comparison group such as keeping
the same iterations and epochs within one loop to our knowledge. On the other hand, this process imposes low search requirements for time and resource. That's to say,  we can draw a  conclusion that our method is at least competitive to PPO if it performs better on benchmarks.  

\subsubsection{Comparison Results}

The final score of  each game  is averaged by three different seeds and the highest is in bold. As Table~\ref{atari_comparison_summary} shows, POP3D outperforms 32 across 49 Atari games in view of the final score, followed by PPO with 11, BASELINE with 5 and TRPO with 1. Interestingly, for games that POP3D score highest, BASELINE score worse than PPO more often than the other way round, which means that POP3D is not just an approximate version of BASELINE.  

For another metric, POP3D wins 20 out of 49 Atari games which matches PPO with 18, followed by BASELINE with 6, and last ranked by TRPO with 5. 

Detailed performance during the whole training process can be found in Figure~\ref{atari-fig}. If we measure the stability of an algorithm by the score variance of different trials, POP3D scores high with good stability across various seeds. And PPO behaves worse  in Game Kangaroo and UpNDown. Interestingly, BASELINE shows a large variance for different seeds for several games such as BattleZone, Freeway, Pitfall and Seaquest.

In sum, POP3D reveals its better capacity to score high and similar fast learning ability in this domain. 
The detailed metric for each game is listed in Table~\ref{final-100-episodes-score} and ~\ref{all-episodes-score}.
\begin{table}
	\begin{center}
		\begin{tabular}{lrrrr}
			\toprule
			Metric & PPO & POP3D & BASELINE & TRPO \\
			\midrule
			$Score_{100}$ & 11 &\textbf{32} & 5 & 1 \\
			$Score_{all}$  & 18 & \textbf{20} &6 &5\\
			\bottomrule
		\end{tabular}
	\end{center}
	\caption{The number of games "won" by each algorithm for Atari game, where the score metric is averaged on three seeds.}
	\label{atari_comparison_summary}
\end{table}

\subsection{Continuous Domain Comparisons}

In this section, we focus on comparisons between POP3D and PPO in Mujoco domain.

\subsubsection{Hyper-parameters}

For PPO, we use the same  hyper-parameter configuration as \cite{2017arXiv170706347S}. Regarding POP3D, we search on two games  three times and select 5.0 as the penalty coefficient. More details about hyper-parameters for PPO and POP3D  are listed in Table~\ref{PPO-Mujoco-game-hyperparameter} and ~\ref{POP3D-Mujoco-game-hyperparameter} of Section~ \ref{hyper-mujoco}. Unlike the Atari domain,  we we utilize the constant learning rate  strategy as \cite{2017arXiv170706347S} in the continuous domain instead of the linear decrease strategy.

\subsubsection{Comparison Results}

The scores are also averaged on three trials and summarized in Table~\ref{comparison_summary}. POP3D occupies 5 out of 7 games on $Score_{100}$, and  3  on $Score_{all}$. Evaluation metrics of both across different games are illustrated in  Table~\ref{final_100_episodes_score_mujoco} and \ref{all-episodes-score-mujoco}. Each algorithm's score performance with iteration steps is shown in Figure~\ref{mujoco-fig}.
\begin{table}
	\begin{center}
		\begin{tabular}{lrr}
			\toprule
			Metric& PPO & POP3D \\
			\midrule
			$Score_{100}$  &2 &  \textbf{5} \\
			$Score_{all}$   & \textbf{4} &3\\
			\bottomrule
		\end{tabular}
	\end{center}
	\caption{The number of games won by each algorithm for Mujoco game, where the score metric is averaged on three seeds.}
	\label{comparison_summary}
\end{table} 

\begin{table}
	\begin{center}
		\begin{tabular}{ >{\raggedright}p{2cm} >{\raggedright\arraybackslash}p{1.4cm}>{\raggedright\arraybackslash}p{1.4cm}>{\raggedright\arraybackslash}p{1.4cm}>{\raggedright\arraybackslash}p{1.4cm}}
			\toprule
			game & POP3D & PPO & BASELINE & TPRO\\
			\midrule
			Alien &\textbf{1510.80}& 1431.17& 1311.23& 1110.40\\
			Amidar & 729.15&\textbf{790.75}& 655.10& 200.56\\
			Assault&\textbf{5400.13}& 4438.82& 1846.75& 1363.46\\ 
			Asterix &\textbf{4310.67}& 3483.17& 3657.67& 2651.33\\
			Asteroids &\textbf{2488.10}& 1605.33& 1615.37& 2205.70\\
			Atlantis &\textbf{2193605.67}& 2140536.33& 1515993.33& 1419104.67\\
			BankHeist &\textbf{1212.23}& 1206.67& 1124.43& 1125.17\\
			BattleZone &\textbf{15466.67}& 14766.67& 14690.00& 15123.33\\
			BeamRider & 4549.00& 2624.19&\textbf{6898.09}& 5073.75\\
			Bowling & 38.99&\textbf{47.27}& 30.48& 31.24\\
			Boxing &\textbf{97.23}& 93.70& 65.33& 50.07\\
			Breakout &\textbf{458.41}& 281.93& 67.70& 40.65\\
			Centipede & 3315.44&\textbf{3565.18}& 3393.93& 3353.14\\
			Chopper-\\
			Command &\textbf{6308.33}& 4872.67& 2676.00& 2286.67\\
			CrazyClimber &\textbf{120247.33}& 105940.00& 98219.67& 87522.33\\
			DemonAttack  &\textbf{61147.33}& 26740.57& 57476.65& 21525.08\\
			DoubleDunk  &\textbf{-7.89}& -11.22& -8.61& -10.04\\
			Enduro  & 459.85&\textbf{698.46}& 518.41& 365.95\\
			FishingDerby &\textbf{28.99}& 17.72& -64.27& -69.64\\
			Freeway &\textbf{21.21}& 21.11& 18.37& 20.89\\
			Frostbite &\textbf{316.87}& 280.30& 280.30& 291.77\\
			Gopher &\textbf{6207.00}& 1791.00& 940.87& 938.27\\
			Gravitar  & 557.17&\textbf{753.50}& 449.00& 495.17\\
			IceHockey & -4.12& -4.83&\textbf{-3.61}& -4.61\\
			Jamesbond & 527.17& 488.17& 685.17&\textbf{901.67}\\
			Kangaroo & 3891.67&\textbf{6845.00}& 1850.00& 1214.67\\
			Krull & 7715.68&\textbf{8329.08}& 7204.95& 4881.65\\
			KungFuMaster &\textbf{33728.00}& 29958.67& 29843.67& 26808.00\\
			Montezuma-\\
			Revenge & 0.00&\textbf{10.67}& 0.67& 0.00\\
			MsPacman & 1683.87&\textbf{1981.50}& 1170.70& 1133.57\\
			NameThisGame &\textbf{6065.63}& 5397.47& 5672.60& 5604.10\\
			Pitfall &\textbf{0.00}& -2.32& -17.26& -43.60\\
			Pong & 20.50&\textbf{20.80}& 20.79& 19.63\\
			PrivateEye & 79.67& 36.50&\textbf{99.67}& 99.33\\
			Qbert &\textbf{15396.67}& 14556.83& 4114.00& 3781.58\\
			Riverraid &\textbf{8052.23}& 7360.40& 7722.00& 6773.67\\
			RoadRunner &\textbf{44679.67}& 36289.33& 43626.33& 24061.33\\
			Robotank & 4.60& 14.15&\textbf{24.60}& 24.18\\
			Seaquest  &\textbf{1807.47}& 1470.60& 1501.47& 926.40\\
			SpaceInvaders &\textbf{1216.15}& 944.63& 814.53& 634.07\\
			StarGunner &\textbf{48984.00}& 33862.00& 47738.00& 33442.67\\
			Tennis &\textbf{-8.32}& -13.74& -19.13& -18.40\\
			TimePilot& 3770.33& 5321.33&\textbf{6278.33}& 5701.00\\
			Tutankham&\textbf{241.21}& 177.58& 135.80& 136.21\\
			UpNDown &\textbf{242701.51}& 153160.66& 11815.87& 10949.53\\
			Venture &\textbf{36.33}& 0.00& 4.00& 0.00\\
			VideoPinball &\textbf{37780.70}& 31577.24& 21438.64& 25095.20\\
			WizardOfWor & 4704.00&\textbf{4886.67}& 3533.67& 3103.00\\
			Zaxxon&\textbf{9472.00}& 5728.67& 1179.67& 4796.67\\
			\bottomrule
		\end{tabular}
	\end{center}
	\caption{Mean final scores (last 100 episodes) of PPO, POP3D, BASELINE and TRPO on Atari games after 40M frames. The results are averaged on three trials.}
	\label{final-100-episodes-score}
\end{table}

In summary, both metrics indicates that POP3D is competitive to PPO in the continuous domain.


\section{Conclusion}

In this paper, we introduce a new reinforcement learning algorithm called POP3D (Policy Optimization with Penalized Point Probability Distance), which acts as a TRPO variant like PPO. Compared with KLD that is an upper bound for the square of total variance divergence between two distributions, the penalized point probability distance is a symmetric lower bound. Besides, it equivalently  expands the optimal solution manifold effectively while encouraging exploration, which is a similar mechanism implicitly possessed by PPO. The proposed method not only possesses several critical improvements from PPO but outperforms with a clear margin on 49  Atari games from the respective of final scores and meets PPO's match as for fast learning ability. 

More interestingly, it not only suffers less from the penalty item setting headache along with TRPO, where is arduous to select one fixed value for various environments, but outperforms fixed KLD baseline from PPO. In summary, POP3D is highly competitive and an alternative to PPO.\\
\clearpage
\bibliographystyle{named}
\bibliography{ijcai19}
\newpage

\appendix
\section{Hyper-parameters}
\subsection{Atari}\label{hyper-atari}
PPO's and  POP3D's hyper-parameters for Mujoco game are respectively listed in Table~\ref{PPO-Atari-game-hyperparameter} and \ref{POP3D-Atari-game-hyperparameter}.
\begin{table}[ht]
	\begin{center}
		\begin{tabular}{lr}
			\toprule
			Hyper-parameter & Value \\
			\midrule
			Horizon (T)    &128 \\
			Adam step-size &2.5 $\times10^{-4}\times\alpha$\\
			Num epochs    & 3 \\
			Mini-batch size    & 32$\times8$\\
			Discount $(\gamma)$     & 0.99\\
			GAE parameter $(\lambda)$      & 0.95 \\
			Number of actors      & 8\\
			Clipping parameter   & 0.1$\times\alpha$ \\
			VF coeff. &1 \\
			Entropy coeff. &0.01 \\
			\bottomrule
		\end{tabular}
	\end{center}
	\caption{PPO's hyper-parameters for Atari game.}
	\label{PPO-Atari-game-hyperparameter}
\end{table}

\begin{table}[ht]
	\begin{center}
		\begin{tabular}{lr}
			\toprule
			Hyper-parameter & Value \\
			\midrule
			Horizon (T)    &128 \\
			Adam step-size &2.5 $\times10^{-4}\times\alpha$\\
			Num epochs    & 3 \\
			Mini-batch size    & 32$\times8$\\
			Discount $(\gamma)$     & 0.99\\
			GAE parameter $(\lambda)$      & 0.95 \\
			Number of actors      & 8\\
			VF coeff. &1 \\
			Entropy coeff. &0.01 \\
			KL penalty coeff. &5.0 \\
			\bottomrule
		\end{tabular}
	\end{center}
	\caption{POP3D's hyper-parameters for Atari game.}
	\label{POP3D-Atari-game-hyperparameter}
\end{table}

\begin{table}[ht]
	\begin{center}
		\begin{tabular}{lr}
			\toprule
			Hyper-parameter & Value \\
			\midrule
			Horizon (T)    &128 \\
			Adam step-size &2.5 $\times10^{-4}\times\alpha$\\
			Num epochs    & 3 \\
			Mini-batch size    & 32$\times8$\\
			Discount $(\gamma)$     & 0.99\\
			GAE parameter $(\lambda)$      & 0.95 \\
			Number of actors      & 8\\
			VF coeff. &1 \\
			Entropy coeff. &0.01 \\
			KL penalty coeff. &10.0 \\
			\bottomrule
		\end{tabular}
	\end{center}
	\caption{BASELINE's hyper-parameters for Atari game.}
	\label{Baseline-Atari-game-hyperparameter}
\end{table}	

\subsection{Mujoco}\label{hyper-mujoco}

PPO's and  POP3D's hyper-parameters for Mujoco game are respectively listed in Table~\ref{PPO-Mujoco-game-hyperparameter} and \ref{POP3D-Mujoco-game-hyperparameter}.
\begin{table}
	\begin{center}
		\begin{tabular}{lr}
			\toprule
			Hyper-parameter & Value \\
			\midrule
			Horizon (T)    &2048 \\
			Adam step-size &3 $\times10^{-4}$\\
			Num epochs    & 10 \\
			Mini-batch size    & 64\\
			Discount $(\gamma)$     & 0.99\\
			GAE parameter $(\lambda)$      & 0.95 \\
			Clipping parameter &0.2 \\
			\bottomrule
		\end{tabular}
		\caption{PPO's hyper-parameters for Mujoco game.}
		\label{PPO-Mujoco-game-hyperparameter}
	\end{center}
\end{table}

\begin{table}
	\begin{center}
		\begin{tabular}{lr}
			\toprule
			Hyper-parameter & Value \\
			\midrule
			Horizon (T)    &2048 \\
			Adam step-size &3 $\times10^{-4}$\\
			Num epochs    & 10 \\
			Mini-batch size    & 64\\
			Discount $(\gamma)$     & 0.99\\
			GAE parameter $(\lambda)$      & 0.95 \\
			KL penalty coeff. &5.0 \\
			\bottomrule
		\end{tabular}
		\caption{POP3D's hyper-parameters for Mujoco game.}
		\label{POP3D-Mujoco-game-hyperparameter}
	\end{center}
\end{table}	

\section{Score Tables and Curves}
Mean scores of various methods for Atari domain are listed in Table~\ref{all-episodes-score}.

\begin{table*}
	\begin{center}
		\begin{tabular}{lrrrr}
			\toprule
			game  & POP3D& PPO  & BASELINE & TRPO\\
			\midrule
			Alien &\textbf{1147.29}& 1115.94& 851.13& 841.08\\
			Amidar & 299.55&\textbf{413.46}& 295.91& 169.12\\
			Assault & 2139.15&\textbf{2168.93}& 1159.50& 971.78\\
			Asterix & 2004.43&\textbf{2102.10}& 1884.68& 1342.83\\
			Asteroids & 1652.48& 1470.46& 1477.71&\textbf{1760.73}\\
			Atlantis & 488134.03&\textbf{596807.27}& 192798.74& 174394.94\\
			BankHeist  & 662.26& 643.94&\textbf{859.25}& 831.95\\
			BattleZone & 11131.44& 9387.77& 11674.30&\textbf{12918.39}\\
			BeamRider & 1965.27& 1460.59&\textbf{3321.25}& 2431.63\\
			Bowling & 37.97&\textbf{39.41}& 33.90& 30.99\\
			Boxing  &\textbf{83.12}& 78.61& 27.92& 23.07\\
			Breakout &\textbf{143.60}& 124.98& 29.99& 26.56\\
			Centipede & 3056.81&\textbf{3344.63}& 3042.48& 3142.22\\
			Chopper-\\
			Command  &\textbf{3269.47}& 3106.14& 1780.38& 1595.82\\
			CrazyClimber &\textbf{97257.52}& 90169.60& 69258.31& 63189.78\\
			DemonAttack & 7611.27& 7180.43&\textbf{9814.42}& 6204.68\\
			DoubleDunk &\textbf{-13.70}& -15.45& -15.93& -14.57\\
			Enduro & 107.84&\textbf{321.20}& 92.59& 140.67\\
			FishingDerby &\textbf{-21.00}& -27.51& -81.90& -81.97\\
			Freeway &\textbf{17.76}& 15.87& 15.93& 17.33\\
			Frostbite &\textbf{276.47}& 267.73& 270.42& 270.57\\
			Gopher &\textbf{1556.29}& 1196.20& 900.74& 875.93\\
			Gravitar & 413.20&\textbf{509.81}& 342.74& 317.86\\
			IceHockey& -4.67& -5.50&\textbf{-4.61}& -5.21\\
			Jamesbond & 358.54& 394.45& 380.91&\textbf{519.01}\\
			Kangaroo & 1614.63&\textbf{2199.74}& 937.98& 566.85\\
			Krull  & 6538.16&\textbf{7195.24}& 4760.66& 3861.87\\
			KungFuMaster & 23253.96&\textbf{23283.31}& 19637.58& 18293.12\\
			Montezuma-\\
			Revenge  & 0.14&\textbf{0.74}& 0.22& 0.12\\
			MsPacman & 1214.09&\textbf{1482.77}& 860.63& 864.84\\
			NameThisGame &\textbf{5353.14}& 5199.37& 4562.32& 4504.67\\
			Pitfall &\textbf{-2.41}& -5.81& -31.27& -33.93\\
			Pong  &\textbf{13.24}& 12.83& 7.20& -2.91\\
			PrivateEye  & 87.37& 52.76& 56.70&\textbf{98.79}\\
			Qbert & 5852.10&\textbf{6744.13}& 1760.92& 1679.03\\
			Riverraid & 5260.89&\textbf{5487.17}& 5220.64& 4549.22\\
			RoadRunner &\textbf{25456.31}& 24688.07& 20385.91& 16269.40\\
			Robotank & 3.08& 8.65& 13.89&\textbf{14.57}\\
			Seaquest &\textbf{1487.84}& 1120.15& 1112.51& 848.47\\
			SpaceInvaders &\textbf{693.26}& 632.17& 552.50& 483.48\\
			StarGunner & 14734.11& 13643.80&\textbf{16288.35}& 13341.23\\
			Tennis &\textbf{-19.86}& -21.80& -21.84& -21.04\\
			TimePilot& 3396.61& 4410.87&\textbf{4718.46}& 4544.68\\
			Tutankham &\textbf{179.96}& 152.72& 103.95& 109.18\\
			UpNDown & 38728.48&\textbf{43208.99}& 5430.22& 7085.02\\
			Venture &\textbf{15.89}& 14.66& 0.57& 0.03\\
			VideoPinball & 27346.44&\textbf{27549.55}& 23998.09& 23705.39\\
			WizardOfWor  & 2340.60&\textbf{2743.40}& 2409.94& 2045.17\\
			Zaxxon  &\textbf{3739.56}& 1813.90& 256.78& 1521.28\\
			\bottomrule
		\end{tabular}
	\end{center}
	\caption{All episodes mean scores of PPO, POP3D, BASELINE and TRPO on Atari games after 40M frames. The results are averaged by three trials.}
	\label{all-episodes-score}
	
\end{table*}

\begin{figure*}
	\begin{center}
		\includegraphics[width=1.5\columnwidth]{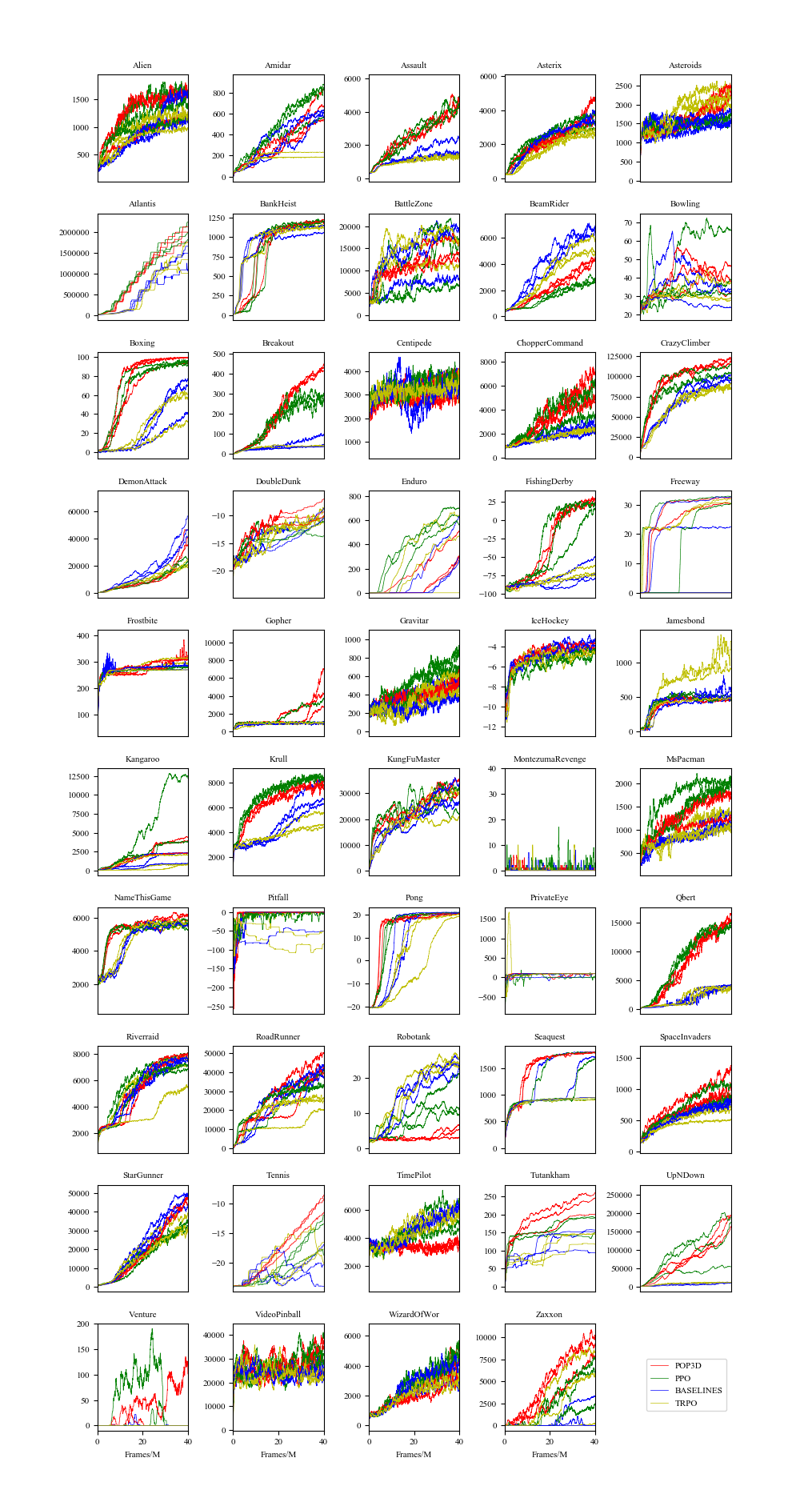}
		\caption{Score curves of three methods on 49 Atari games within 40 million frame steps.}
		\label{atari-fig}
	\end{center}
\end{figure*}

\begin{figure*}
	\begin{center}
		\includegraphics[width=1.15\columnwidth]{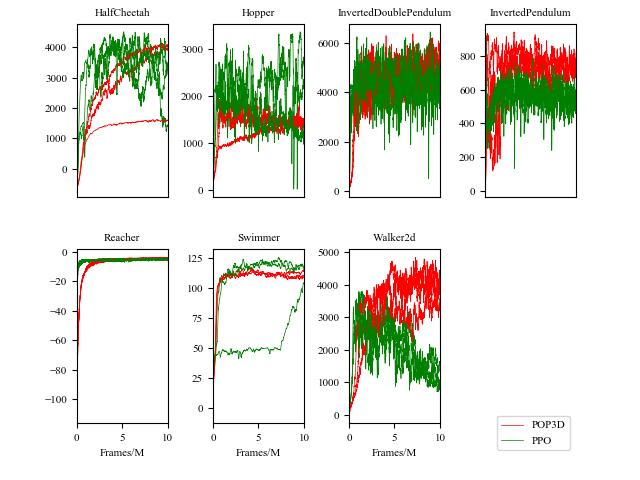}
		\caption{Score curves of three methods on 7 Mujoco games within 10 million frame steps.}
		\label{mujoco-fig}
	\end{center}
\end{figure*}

\begin{table*}
	\begin{center}
		\begin{tabular}{lrrr}
			\toprule
			game & PPO & POP3D  \\
			\midrule
			HalfCheetah & 2726.03 &\textbf{3184.54}  \\
			Hopper & \textbf{2027.21} &1452.09 \\
			InvertedDoublePendulum & 4455.03 &\textbf{4907.64}  \\
			InvertedPendulum & 544.02 &\textbf{741.94}  \\
			Reacher& -5.00 &\textbf{-4.29}  \\ 
			Swimmer & \textbf{111.88} &111.08  \\ 
			Walker2d & 1112.25 &\textbf{3966.01} \\ 
			\bottomrule
		\end{tabular}
	\end{center}
	\caption{Mean final scores (last 100 episodes) of PPO ,POP3D on Mujoco games after 10M  frames. The results are averaged on three trials.}
	\label{final_100_episodes_score_mujoco}
\end{table*}

\begin{table*}
	\begin{center}
		\begin{tabular}{lrr}
			\toprule
			game & PPO & POP3D\\
			\midrule
			HalfCheetah & \textbf{3250.22} &2373.30  \\
			Hopper & \textbf{1767.14} &1257.72  \\
			InvertedDoublePendulum & \textbf{3684.92} &2561.77  \\
			InvertedPendulum  & 531.77 &\textbf{552.98}  \\
			Reacher & \textbf{-5.94} &-8.05  \\
			Swimmer  & 94.01 &\textbf{108.27}  \\
			Walker2d  & 1770.37 &\textbf{2439.54}  \\
			\bottomrule
		\end{tabular}
	\end{center}
	\caption{All episodes mean scores of PPO ,POP3D on Mujoco games after 10M frames. The results are averaged by three trials.}
	\label{all-episodes-score-mujoco}
\end{table*}

\end{document}